\documentclass[letterpaper]{article} 
\usepackage{aaai2027}  
\usepackage[hyphens]{url}  
\usepackage{graphicx} 
\usepackage{natbib}  
\usepackage{caption} 
\usepackage{algorithm}
\usepackage{algorithmic}

\usepackage{newfloat}
\usepackage{listings}
\DeclareCaptionStyle{ruled}{labelfont=normalfont,labelsep=colon,strut=off} 
\floatstyle{ruled}
\newfloat{listing}{tb}{lst}{}
\floatname{listing}{Listing}

\usepackage{booktabs}
\usepackage{amsthm}
\newtheorem{theorem}{Theorem}
\usepackage{amsmath}
\usepackage{amsfonts}
\usepackage{amssymb}
\usepackage{nicefrac}
\usepackage{microtype}
\usepackage{pifont}   
\usepackage{multirow}
\usepackage{tabularx}
\usepackage{colortbl}
\usepackage{makecell}
\definecolor{ourslitebg}{RGB}{234,243,255}
\definecolor{oursprobg}{RGB}{235,248,240}
\usepackage[utf8]{inputenc} 
\nocopyright

\title{LeapTalk: Breaking the Latency–Quality Trade-off in Talking Head Generation}

\author{
    Rongxiang Zhang\textsuperscript{\rm 1,\rm 2}, 
    Songhua Liu\textsuperscript{\rm 1}\corresponding
}
\affiliations{
    \textsuperscript{\rm 1}School of Artificial Intelligence, Shanghai Jiao Tong University\\
    \textsuperscript{\rm 2}Harbin Institute of Technology\\

    zrongxiang1@gmail.com, liusonghua@sjtu.edu.cn
}

\begin{document}

\maketitle

\begin{abstract}
Long-form and real-time talking-head generation remains challenging due to a latency-quality trade-off: inefficient multi-step diffusion prohibits streaming generation, whereas real-time autoregressive approaches suffer from error accumulation and identity drift. 
To address this drawback, we propose LeapTalk, a novel framework that achieves stable and real-time talking-head generation with a single forward step, scaling to arbitrarily long videos. 
At the heart of our approach lies a single-step bridge distillation scheme. 
On the one hand, departing from the conventional \emph{noise-to-data} paradigm, we introduce a \emph{data-to-data} transport formulation based on a Brownian bridge. 
Anchored by a persistent reference, this strategy effectively mitigates identity drift and enhances long-term temporal stability. 
On the other hand, to enable smooth knowledge transfer from a pre-trained diffusion teacher to the student bridge model, we explore a heterogeneous distillation framework with an SNR-aligned time transformation $\Phi(\tau)$, which bridges the functional discrepancy between the two models.  
Moreover, we propose an audio-driven classifier-free guidance mechanism to maintain fine-grained lip synchronization under extreme step reduction. 
Extensive experiments demonstrate that our method achieves high-fidelity and temporally consistent video generation with only \textbf{1 step} at up to \textbf{\textit{200 FPS}}, significantly outperforming existing approaches in both efficiency and stability.  \textbf{Project Page:} \url{https://zhangrongxiang.github.io/leaptalk-page/}
\end{abstract}


\section{Introduction}
Audio-driven avatar video generation seeks to produce photorealistic human-centric videos conditioned on a reference image and an audio signal, where facial expressions and body motions are temporally synchronized with speech ~\citep{tu2025stableavatar,weng2026eartalkingendtoendgptstyleautoregressive}.
This capability enables a wide range of applications, including digital humans, virtual assistants, and content creation in film and media. However, translating these capabilities into practical, real-time systems remains challenging.

Despite promising performance, we notice that existing methods face a fundamental trade-off between \textit{latency} and \textit{generation quality.}
 \textbf{ (a) High latency and limited length.} Diffusion-based approaches~\citep{gao2025wan,zhong2025anytalker,guo2024liveportrait,yang2025infinitetalk,tu2025stableavatar}. remain computationally expensive due to iterative denoising, typically generating only short clips (5--10s). 
Although methods like StableAvatar~\citep{tu2025stableavatar} enable long-form synthesis, they operate in an offline manner and cannot support real-time interaction.
\textbf{(b) Error accumulation.} To improve efficiency, autoregressive methods~\citep{shen2025soulx,yu2026soulxflashheadoracleguidedgenerationinfinite,huang2025selfforcing,mazumdar2026temposyncdiffdistilledtemporallyconsistentdiffusion} convert multi-step diffusion into few-step models via distillation techniques and adopt streaming generation. 
However, their reliance on previously generated frames introduces exposure bias, causing errors to accumulate over time and leading to identity corruption and degraded visual quality in long sequences~\citep{zhu2026causalforcingautoregressivediffusion,yu2026soulxflashheadoracleguidedgenerationinfinite,chern2025livetalk}.
\textbf{(c) Loss of facial details.}  While distribution matching distillation (DMD)~\citep{yin2024one} improves efficiency and partially mitigates exposure bias~\citep{zhu2026causalforcingautoregressivediffusion,chern2025livetalk}, it reduces sampling steps and sacrifices high-frequency details. 
This results in degraded facial dynamics, particularly inaccurate or ambiguous lip synchronization under few-step settings. In summary, diffusion processes achieve high visual fidelity but suffer from prohibitive latency, while  autoregressive streaming approaches often compromise long-term stability and fine-grained details.
\begin{figure*}[t]
  \includegraphics[width=\textwidth]{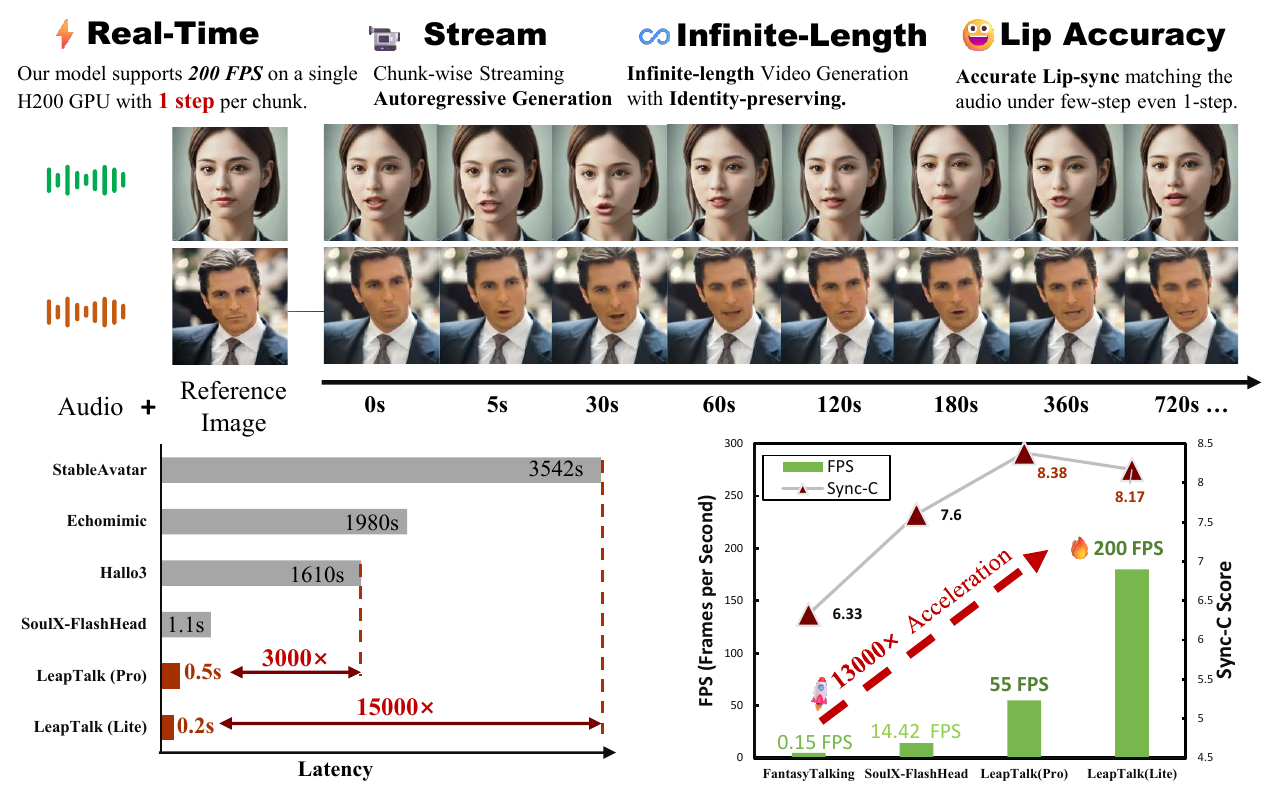}
\caption{\textbf{Overview of LeapTalk.} Given audio and a reference image, our method enables open-ended streaming talking-head generation with consistent identity. It achieves \textbf{1-step} inference per chunk at up to \textbf{200 FPS}, delivering up to $\mathbf{15000\times}$ speedup while maintaining strong lip-sync accuracy.}
  \label{fig:teaser}
\end{figure*}

To address this trade-off, we propose \textbf{LeapTalk}, a bridge-based autoregressive framework for real-time talking-head generation. 
LeapTalk is, to our knowledge, the first approach that enables stable open-ended and real-time streaming talking head generation, where each video chunk is synthesized in  \textbf{1 step} while preserving accurate lip synchronization, fine-grained facial details, and strong identity consistency.

Our key insight is to reformulate talking-head generation as a \emph{data-to-data} transport process via a Brownian bridge model, implemented through a novel \emph{Bridge Forcing} paradigm.
By introducing a persistent visual anchor $\mathcal{I}$ as the starting point of generation, the trajectory is grounded to a fixed identity reference, substantially mitigating long-term drift and error accumulation.
To enable efficient training under heterogeneous generative paradigms, we propose a \emph{heterogeneous DMD} scheme that distills a flow-matching~\citep{lipman2023flowmatchinggenerativemodeling} teacher into a bridge student, where a time transformation function aligns their signal-to-noise ratios for score matching. 
Finally, to preserve fine-grained facial dynamics under few-step generation, we introduce an \emph{audio-driven classifier-free guidance (CFG) augmentation}, which enhances lip-sync accuracy and enhance motion diversity, prevents detail collapse and static motion in one-step generation.
In summary, our key contributions include:
\begin{itemize}
\item We propose \textbf{LeapTalk}, a real-time, streaming talking-head generation framework that achieves \textbf{stable one-step inference} with up to \textbf{200 FPS} on a single GPU, while maintaining high visual fidelity and long-term identity consistency.

\item We reformulate talking-head generation as a \textbf{\emph{data-to-data}} transport problem via a Brownian bridge mathematical formulation, and introduce Bridge Forcing to reduce error accumulation and identity drift in autoregressive generation.

\item We propose a heterogeneous DMD framework to distill a flow-matching teacher into a Brownian-bridge student including two key components: \textbf{(i) an SNR-aligned time transformation $\Phi(\tau)$} that resolves trajectory mismatch and enables well-defined score matching, and \textbf{(ii) an audio-driven CFG augmentation mechanism} that preserves motion quality under few-step generation.
\item  Extensive experiments demonstrate that our bridge-based method achieves competitive or superior performance across multiple benchmarks, while significantly improving efficiency and enabling high-quality, open-ended generation in the extreme one-step regime.
\end{itemize}

\section{Preliminaries}
\subsection{Probability Path Modeling and Rectified Flow}
Probability-path models~\citep{NEURIPS2021_940392f5,lipman2023flowmatchinggenerativemodeling,liu2022flowstraightfastlearning} generate data by transporting a source distribution $p_0$ to a target distribution $p_1$ over continuous time. A common formulation is the SDE~\citep{song2021scorebasedgenerativemodelingstochastic}
\begin{equation}\label{eq:sde}
dX_t = v(X_t, t)\, dt + \sigma(t)\, dW_t,\quad t \in [0,1],
\end{equation}
where $X_0 \sim p_0$, $X_1 \sim p_1$, $v$ is the velocity field, $\sigma(t)$ controls diffusion strength, and $W_t$ is Brownian motion.

Rectified Flow~\citep{liu2022flowstraightfastlearning} is the deterministic limit of Eq.~\ref{eq:sde} with $\sigma(t)=0$. It typically transports Gaussian noise $x_0\sim\mathcal{N}(0,I)$ to data $x_1\sim p_1$ along
\begin{equation}
    x_t=(1-t)x_0+tx_1.
\end{equation}
The corresponding target velocity is constant: $u_t=x_1-x_0$. 

\subsection{Brownian Bridge}
 Instead of starting from pure noise, Brownian bridge connects two fixed endpoints $(x_0,x_1)$ while retaining stochasticity along the path. Conditioned on $x_0 \sim p_0$ and $x_1 \sim p_1$, the intermediate state follows
\begin{equation}\label{bb}
X_t \mid _{(x_0, x_1)} \sim \mathcal{N}\left((1-t)x_0+t x_1,\; t(1-t)I\right).
\end{equation}
Equivalently, samples can be written as
\begin{equation}{\label{bl}}
    X_{t} = (1 - t) X_0 + t X_1 + \sqrt{t(1 - t)} \epsilon, \quad t \in [0,1],
\end{equation}
where $\epsilon \sim \mathcal{N}(0,I)$. The bridge preserves both endpoints while injecting uncertainty only in the middle of the trajectory, making it effective for paired data-to-data problems~\citep{tan2025vision}. Its conditional velocity field is
\begin{equation}
u_t(X_t \mid x_0, x_1) = \frac{x_1 - X_t}{1 - t}.
\end{equation}
A neural network $v_\theta(X_t,t)$ can then be trained by velocity matching:
\begin{equation}
\label{eq:bridgeloss}
\mathcal{L}(\theta) = \mathbb{E}_{(x_0, x_1),\, t,\, X_t} 
\left[
\left\| v_\theta(X_t, t) - \frac{x_1 - X_t}{1 - t} \right\|^2
\right],
\end{equation}
where $t \sim \mathcal{U}(0,1)$ and $X_t$ is sampled from Eq.~\ref{bl}. 

Prior bridge-based generative models~\citep{wang2025framebridge, tan2025vision} have demonstrated the effectiveness of Brownian bridge formulations, but they still rely on multi-step sampling and offline inference. In contrast, LeapTalk achieves bridge-based stable one-step streaming generation with strong long-term consistency via heterogeneous distillation.

\section{Methods}

\subsection{Reference-Anchored  Bridge Forcing}
\label{sec:reference_anchored_ar}

\noindent\textbf{Limitations of Noise-to-Data Autoregressive Diffusion.}
Following the formulations of recent works~\citep{huang2025selfforcing, zhu2026causalforcingautoregressivediffusion}, standard autoregressive (AR) video diffusion models factorize the joint distribution of a sequence $X^{1:L}$ into a product of conditional distributions:
\begin{equation}
    p(X^{1:L}) = \prod_{k=1}^{K} p(C_k \mid C_{<k}, \mathbf{c}_k),
\end{equation}
where each chunk $C_k$ is generated via noise-to-data flow matching from Gaussian noise conditioned on $C_{<k}$. 
Such repeated reconstruction from noise leads to error accumulation, causing identity drift and visual degradation in long sequences.

\noindent\textbf{Motivation.}
In talking-head generation, frames share identity and differ mainly in motion, which motivates modeling generation as a \textbf{\emph{data-to-data}} transport from the reference image $\mathcal{I}$ to target frames.

We propose \textbf{Bridge Forcing}, which replaces noise-to-data reconstruction with a reference-anchored Brownian Bridge. The AR process becomes
\begin{equation}
    p(X^{1:L}) = \prod_{k=1}^{K} p(C_k \mid C_{<k}, \mathcal{I}, \mathbf{c}_k),
\end{equation}
where each chunk following the bridge trajectory (Eq.~\ref{bl}) from the reference image $\mathcal{I}$ to target frame $X_1$ :
\begin{equation}
{X}^{(n)}_t = (1-t)\mathcal{I} + t X_1 + \sqrt{t(1-t)}\,\epsilon,\quad \epsilon \sim \mathcal{N}(0,I).
\end{equation}

\noindent\textbf{Chunk Construction under Bridge Forcing.}
For each chunk, we initialize the sequence by repeating the reference image $\mathcal{I}$ across all frames, ensuring a clean and identity-consistent starting point. To maintain temporal continuity, we replace the first $K$ frames with the last $K$ frames from the previous chunk, which serve as motion prefix. The resulting chunk input is constructed as
\begin{equation}
C_k^{\text{in}} = \big[\, C_{k-1}^{(L-K+1:L)},\; \underbrace{\mathcal{I}, \ldots, \mathcal{I}}_{L-K} \,\big],
\end{equation}
where $C_{k-1}^{(L-K+1:L)}$ denotes the last $K$ frames of the previous chunk. This design preserves motion continuity while anchoring identity at every chunk.
Formally, let $\mathbf{X}^{(n-1)}$ denote the previous chunk and $\tilde{\mathbf{X}}^{(n)}$ the initialized sequence for the current chunk. The constructed input $\mathbf{X}^{(n)}$ is given by
\begin{equation}
\mathbf{X}^{(n)} = \big[
\mathbf{X}^{(n-1)}_{T-K:T},\;
\underbrace{\mathcal{I}, \dots, \mathcal{I}}_{T-K}
\big],
\end{equation}
where $T$ is the chunk length. This prefix replacement design enforces identity preservation while enabling smooth temporal stitching between consecutive chunks.
 \begin{figure}[h!]
    \centering
    \includegraphics[width=1\linewidth]{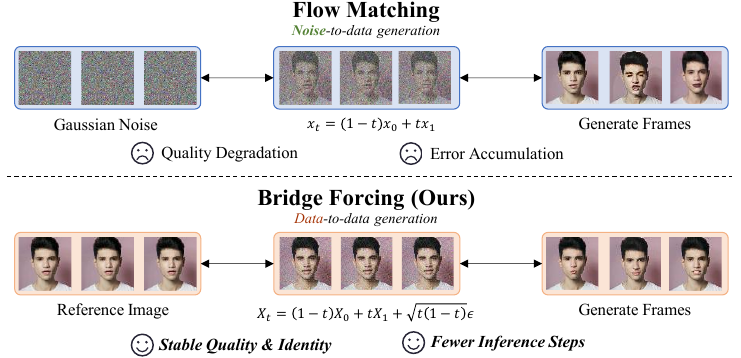}
    \caption{\textbf{Comparison between conventional noise-to-data flow matching and our  Bridge Forcing paradigm.} The identity consistency is preserved along bridge process, while error accumulates in the forward diffusion process.}
    \label{fig:bridge_comparison}
\end{figure}

\subsection{Heterogeneous Distribution Matching Distillation}
Real-time talking-head generation requires autoregressive generation with few-step (ideally 1 step) synthesis. We propose a heterogeneous DMD (See Figure ~\ref{fig:distill}) scheme with SNR-aligned time transformation and audio-driven CFG enables high-quality generation.

\begin{figure*}[ht!]
    \centering
    \includegraphics[width=1\linewidth]{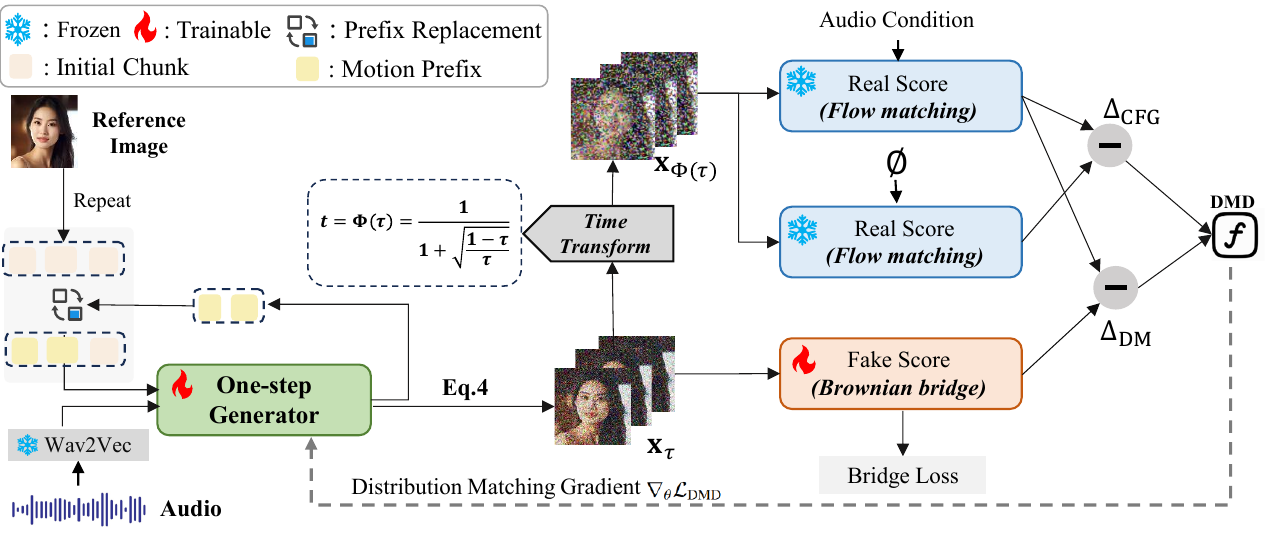}
    \caption{\textbf{Distillation pipeline of our approach.} 
The one-step student generates frames conditioned on a static reference. 
To enable stable distillation across heterogeneous processes, we apply a time transformation $t=\Phi(\tau)$ to align noise levels between the flow-matching teacher and the bridge-based student, allowing consistent score supervision and a well-defined DMD objective.}
    \label{fig:distill}
\end{figure*}
\noindent\textbf{Time Transformation. }Standard Distribution Matching Distillation minimizes the Kullback–Leibler (KL) divergence, and the gradient takes the following form:
\begin{equation}\label{eq:dmd_grad}
\nabla_{\theta} \mathcal{L}_{\text{DMD-theory}} = \mathbb{E}_{z_t, \tau, \mathbf{x}_{\tau}} \left[ 
- \left( s_{\text{cond}}^{\text{real}}(\mathbf{x}_{\tau}) - s_{\text{cond}}^{\text{fake}}(\mathbf{x}_{\tau}) \right) 
\frac{\partial G_{\theta}}{\partial \theta} 
\right].
\end{equation}
where $s_{\text{cond}}^{\text{real}}$ and $s_{\text{cond}}^{\text{fake}}$ denote the teacher and student score functions evaluated at the same noise level. And $G_\theta$ is the one-step generator.

This formulation assumes that teacher and student share an identical forward process. In our setting, this assumption is violated: the teacher follows a flow-matching trajectory, while the student is defined on a Brownian bridge, resulting in mismatched noise levels at the same timestep. Consequently, the same timestep corresponds to different signal-to-noise ratios (SNRs), leading to inconsistent noise levels and rendering the score difference ill-defined.

To resolve this issue,  we construct a time transformation function $\Phi(\tau)$ to align the two processes through SNR matching. 

\begin{theorem}[SNR-aligned Time Transformation Function]
Let $\text{SNR}_{\text{teacher}}(t)$ and $\text{SNR}_{\text{student}}(\tau)$ denote the SNR of the flow-matching teacher and the bridge student, respectively. There exists a monotonic time transformation $t = \Phi(\tau)$ that aligns the two processes such that
\begin{equation}
\text{SNR}_{\text{teacher}}(t) = \text{SNR}_{\text{student}}(\tau).
\end{equation}
The corresponding closed-form solution is given by
\begin{equation}
t = \Phi(\tau) = \frac{1}{1 + \sqrt{\frac{1-\tau}{\tau}}}.
\end{equation}
\end{theorem}
The derivation is provided in Supplementary Material.

This mapping enables consistent score evaluation across heterogeneous processes. Given an input, the student first predicts a one-step output $\hat{X}_1 = G_{\theta}(\cdot)$. We then construct intermediate states $\mathbf{x}_{\tau}$ along the Brownian bridge Eq.~\ref{bb}. 
The corresponding teacher input is obtained via time alignment:
\begin{equation}
\mathbf{x}_{\Phi(\tau)} \sim \text{teacher process at } t=\Phi(\tau),
\end{equation}
ensuring both scores are evaluated under the same SNR.

\noindent\textbf{Audio-driven Classifier-free Guidance (CFG) Augmentation.}
While time alignment resolves the inconsistency across trajectories, few-step distillation—especially in the one-step regime—tends to suppress high-frequency dynamics, leading to degraded lip-sync and static motion. To compensate for this, we introduce a modification by \textbf{replacing the teacher score with a Classifier-free Guidance score:}
\begin{equation}\label{eq:new_grad}
\nabla_{\theta} \mathcal{L} = \mathbb{E} \left[ 
- \left( s_{\text{cfg}}^{\text{real}}(\mathbf{x}_{\Phi(\tau)}) - s_{\text{cond}}^{\text{fake}}(\mathbf{x}_{\tau}) \right) 
\frac{\partial G_{\theta}}{\partial \theta} 
\right].
\end{equation}

Here, $s_{\text{cfg}}^{\text{real}}$ is defined as an audio-driven classifier-free guidance 
$s_{\text{cfg}}^{\text{real}} 
= s_{\text{cond}}^{\text{real}} 
+ (\alpha - 1)\left( s_{\text{cond}}^{\text{real}} - s_{\text{uncond}}^{\text{real}} \right)$, 
where the unconditional branch is obtained using zero audio input, and $\alpha$ controls the guidance strength.

Substituting $s_{\text{cfg}}^{\text{real}}$ into Eq.~(\ref{eq:new_grad}), we obtain the final DMD gradient:
\begin{equation}
\label{eq:dmd_final}
\begin{aligned}
\nabla_{\theta} \mathcal{L}_{\text{DMD}}
&= \mathbb{E}\left[
-\left(\Delta_{\text{DM}}+(\alpha-1)\Delta_{\text{CFG}}\right)
\frac{\partial G_{\theta}}{\partial \theta}
\right], \\
\Delta_{\text{DM}}
&= s_{\text{cond}}^{\text{real}}(\mathbf{x}_{\Phi(\tau)})
- s_{\text{cond}}^{\text{fake}}(\mathbf{x}_{\tau}), \\
\Delta_{\text{CFG}}
&= s_{\text{cond}}^{\text{real}}(\mathbf{x}_{\Phi(\tau)})
- s_{\text{uncond}}^{\text{real}}(\mathbf{x}_{\Phi(\tau)}).
\end{aligned}
\end{equation}

Meanwhile, the fake score network is trained to match the  the student generator under the residual bridge trajectory. Specifically, given a generated sample, we construct intermediate states $\mathbf{X_{\tau}}$ constructed from Eq.~\ref{bb} and optimize the fake network $v_{\theta}(X_{\tau}, \tau)$ with the bridge loss in Eq.~\ref{eq:bridgeloss}
where $X_1$ denotes the one-step output of the generator.
This objective ensures that the fake network can serve as the student-side score function for subsequent distillation.

\noindent\textbf{Bridge Training.}
Given a reference $\mathcal{I}$, we construct the bridge state $\mathbf{x}_{\tau}$ and predict the velocity to obtain the one-step reconstruction:
$
\hat{\mathbf{x}}_0 = \mathbf{x}_{\tau} - \tau \, v_\theta(\mathbf{x}_{\tau}).
$
The student is trained to match a target endpoint $\mathbf{x}_0$ generated by the multi-step teacher. To emphasize facial and lip regions, we introduce a spatial weighting matrix:
$
\mathbf{W} = 1 +  \mathbf{M}_{\text{face}} +  \mathbf{M}_{\text{lip}},
$
where $\mathbf{M}_{\text{face}}$ and $\mathbf{M}_{\text{lip}}$ are binary masks obtained via Mediapipe~\citep{mediapipe}. 
The weighted reconstruction loss is:
$
\mathcal{L}_{\text{rec}} = \mathbb{E} \left[ 
 \mathbf{W} \odot \left\| \hat{\mathbf{x}}_0 - \mathbf{x}_0 \right\|_2^2 
\right].
$
Following prior works~\citep{chadebec2025lbm}, we additionally apply a Learned Perceptual Image Patch Similarity (LPIPS)~\citep{Zhang_2018_CVPR} loss $\mathcal{L}_{\text{perc}}$. 
The final objective is:
\begin{equation}
\mathcal{L} =  \mathcal{L}_{\text{DMD}} + \mathcal{L}_{\text{rec}} + \lambda_{\text{perc}} \mathcal{L}_{\text{perc}}.
\label{eq:final_loss}
\end{equation}

To bridge the gap between training and inference, we adopt an autoregressive self-rollout strategy in Self-Forcing~\citep{huang2025selfforcing} where the model sequentially generates video chunks conditioned on its previously generated history under the bridge formulation. 


\noindent\textbf{Model Architecture.}
Both the teacher and student model are built upon the Wan2.1-T2V-1.3B~\citep{wan2025wan} backbone, adopting a Diffusion Transformer to model spatio-temporal dynamics in latent space. To balance visual fidelity and efficiency, we employ two autoencoder variants: a 3D Conv-based WanVAE for the \textbf{Pro model}, and a lightweight 2D Conv-based TAEHV~\citep{BoerBohan2025TAEHV} for the \textbf{Lite model}, which reduces computation and memory with only minor quality degradation. For audio conditioning,  frame-aligned features are extracted from a pretrained Wav2Vec~\citep{baevski2020wav2vec} audio encoder and injecting them into the DiT via cross-attention.
\section{Experiments}
\noindent\textbf{Implementation Details.}
In the  distillation stage, the learning rates are set to $1 \times 10^{-4}$ for the generator and $2 \times 10^{-6}$ for the fake score network, with an update ratio of 1:5. To better simulate long-term error accumulation, the generator produces up to $N=2$ history chunks during distillation. $\lambda_{\text{perc}}$ is set to $4$. The model is trained on the VividHead Dataset~\citep{yu2026soulxflashheadoracleguidedgenerationinfinite}, from which we choose the first frame as the reference image input. Unless otherwise specified, all experiments are performed on a single NVIDIA H200 GPU with a batch size of 1 with 10K iterations.  
\begin{table*}[t]
\centering
\small
\setlength{\tabcolsep}{2.4pt}
\begin{tabular}{@{}l c ccccccc ccccccc@{}}
\toprule
& & \multicolumn{7}{c}{HDTF} & \multicolumn{7}{c}{CelebV-HQ} \\
\cmidrule(lr){3-9} \cmidrule(lr){10-16}
Model & NFE 
& FID$\downarrow$ & FVD$\downarrow$ & Sync-C$\uparrow$ & Sync-D$\downarrow$ & IQA$\uparrow$ & ASE$\uparrow$ & FPS$\uparrow$
& FID$\downarrow$ & FVD$\downarrow$ & Sync-C$\uparrow$ & Sync-D$\downarrow$ & IQA$\uparrow$ & ASE$\uparrow$ & FPS$\uparrow$ \\
\midrule

StableAvatar     & 50 & 176 & 329 & 8.11 & 8.05 & 6.51 & 2.69 & 0.42 & 318 & 492 &4.73 & 8.61 & 5.42 & 3.30 & 0.42 \\
Echomimic        & 30 & 722 & 981 & 5.32 & 10.02 & 6.16 & 2.29 & 0.52 & 642 & 1885 & 1.09 & 11.02 & 0.15 & 3.19 & 0.52 \\
Hallo3           & 51 & 871 & 972 & 7.14 & 9.23 & 6.24 & 2.60 & 0.30 & 842 & 1104 & 2.49 & 8.85 & 5.23 & 3.22 & 0.29 \\
FantasyTalking   & 30 & 459 & 884 & 6.33 & 9.41 & 6.13 & 2.45 & 0.15 & 544 & 1429 & 1.84 & 8.61 & 5.21 & 3.18& 0.15 \\
OmniAvatar       & 50 & 168 & 623 & 3.10 & 12.36 & 6.32 & 2.68 & 0.18 & 351 & 923 & 1.25 & 10.03 & 5.39 & 3.27 & 0.16 \\
Soulx-Flashhead  & 4  & \underline{30}  & 452 & 8.07 & 8.24 & \textbf{6.60} & \textbf{2.95} & 14.42 & 71  & 642 & 4.77 & 8.23 & \underline{5.57} & \textbf{3.35} & 14.42 \\
\rowcolor{ourslitebg}
\textbf{OURS (Lite)}      & {\textbf{1}}  & 38  & \underline{285} & \underline{8.14} & \underline{7.89} & 6.22 & 2.70 & \textbf{200} & \underline{47}  & \underline{456} & \underline{4.80} & \underline{8.22} & 5.51 & 3.33 & \textbf{200} \\
\rowcolor{oursprobg}
\textbf{OURS (Pro) }      & \textbf{1}  & \textbf{21} & \textbf{197} & \textbf{8.38} & \textbf{7.69} & \underline{6.53} & \underline{2.74} & \underline{55} 
& \textbf{42} & \textbf{370} & \textbf{5.05} & \textbf{8.21} & \textbf{5.58} & \underline{3.34} & \underline{55} \\

\bottomrule
\end{tabular}
\caption{Comparison on HDTF and CelebV-HQ datasets.}
\label{tab:comparison}
\end{table*}

\noindent\textbf{Evaluation Metrics.}
We compare LeapTalk with recent baselines, including SoulX-FlashHead~\citep{yu2026soulxflashheadoracleguidedgenerationinfinite}, StableAvatar~\citep{tu2025stableavatar}, EchoMimic~\citep{meng2025echomimicv3}, Hallo3~\citep{cui2025hallo3}, FantasyTalking~\citep{wang2025fantasytalking}, and OmniAvatar~\citep{gan2025omniavatar}, on 80 videos sampled from HDTF~\citep{zhang2021flow} and CelebV-HQ~\citep{zhu2022celebvhq}.
For a fair comparison, all baselines are run under the same hardware and evaluation configuration, using identical input videos, resolution settings, and inference batch size. 
We use FID~\citep{heusel2017gans} and FVD~\citep{unterthiner2019fvd} for frame quality, Sync-C/Sync-D~\citep{chung2016out} for audio-visual synchronization, and FPS for inference speed. Long-term identity consistency is measured by tracking DINOv2~\citep{oquab2023dinov2} similarity between each generated frame and the reference image.
 \begin{figure}[h!]
    \centering
    \includegraphics[width=1\linewidth]{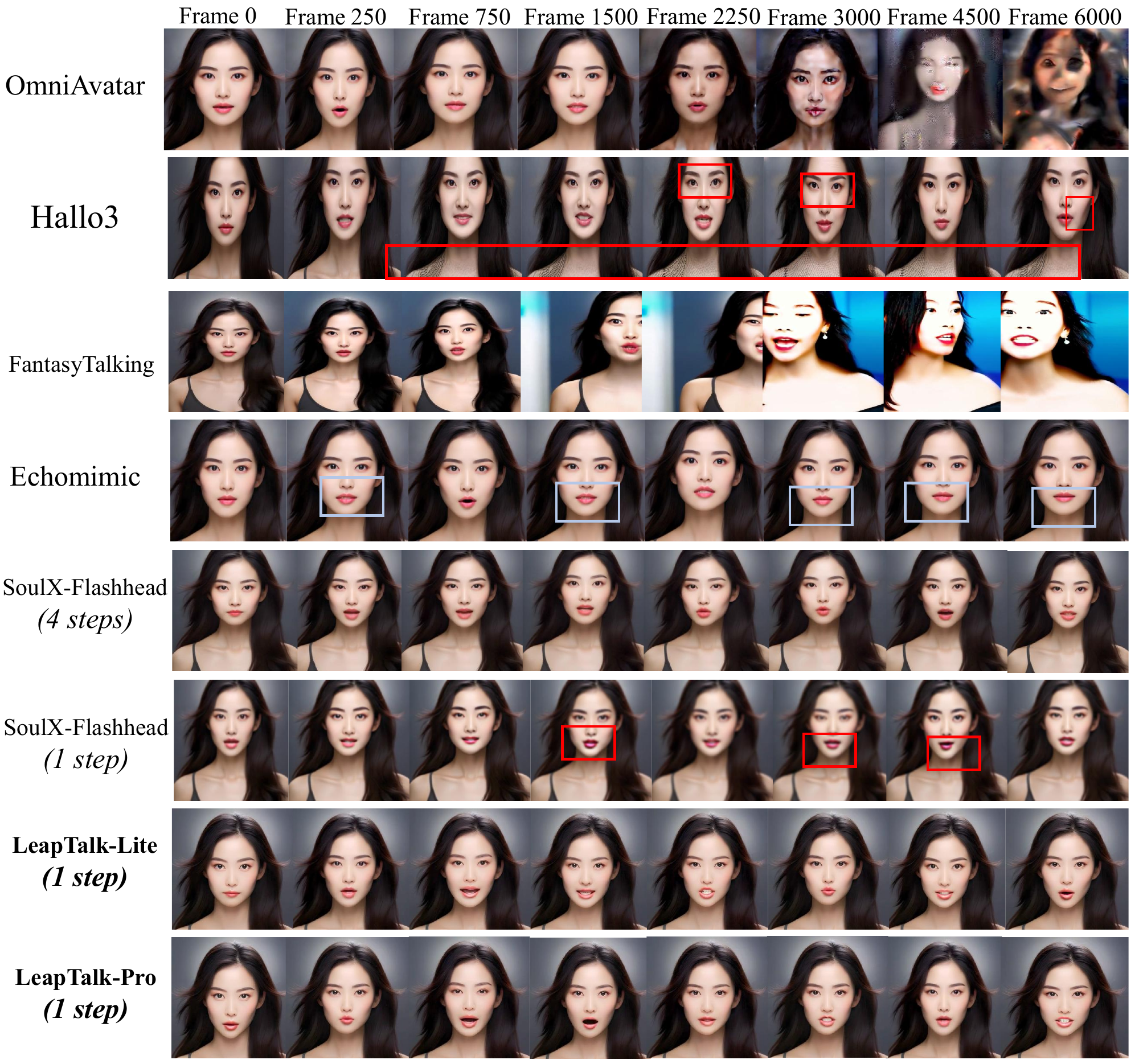}
    \caption{\textbf{Qualitative comparison on long-video streaming generation.} Frames are sampled from streaming rollouts at increasing time indices. The \textcolor{red}{red boxes} mark identity inconsistency and drift, and the \textcolor{blue}{blue boxes} mark slight or incorrect lip-sync. LeapTalk maintains stable identity and lip motion as generation proceeds.}
    \label{fig:quali}
\end{figure}

\noindent\textbf{Quantitative Results.}
Table~\ref{tab:comparison} shows that LeapTalk achieves strong quality while using only \textbf{1 NFE}. LeapTalk Pro obtains the best FID/FVD (21/197 on HDTF and 42/370 on CelebV-HQ) and the strongest lip-sync scores, while LeapTalk Lite reaches up to \textbf{200 FPS}. Figure~\ref{fig:dino} further evaluates identity consistency during streaming generation by tracking DINO similarity to the reference image. As time progresses, our model always keeps a high and stable similarity curve, whereas baselines such as OmniAvatar degrade noticeably. We also provide motion diversity variance analysis in the supplementary material. This sustained stability supports long-duration generation and suggests that our bridge method can be extended to open-ended and infinite generation with both high speed and stability.
\begin{figure}
    \centering
    \includegraphics[width=1\linewidth]{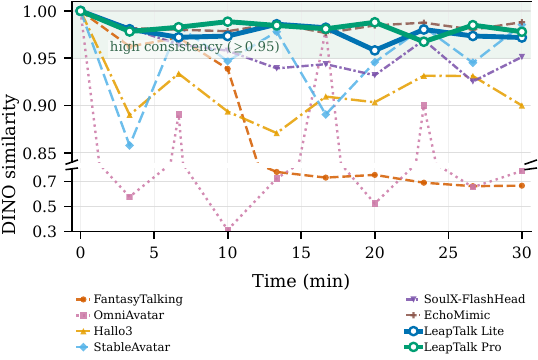}
    \caption{\textbf{DINO similarity over video time.} We plot the similarity between each generated frame and the reference image as streaming generation progresses. A flatter and higher curve indicates less identity drift over long-duration generation.}
    \label{fig:dino}
\end{figure}

\noindent\textbf{Qualitative Results.}
Figure~\ref{fig:quali} reports qualitative comparisons under long-video streaming generation. As the frame index grows, baseline methods accumulate visible artifacts, drift from the reference identity, and exhibit unstable lip synchronization. In contrast, LeapTalk preserves identity, facial structure, and lip motion throughout the streaming process, showing stronger long-duration generation stability.
Furthermore, compared with 4-step SoulX-FlashHead, our method maintains comparable quality with only \textbf{1 step} per chunk, while SoulX-FlashHead degrades significantly in the 1-step setting with noticeable identity drift.  This highlights the advantage of our bridge-based formulation for stable and efficient streaming open-ended generation.

\subsection{Ablation Study}
We ablate the Pro model to quantify the contribution of each component. For the Brownian Bridge ablation, we keep the same model architecture and remove only the bridge formulation. As shown in Table~\ref{tab:ablation} and Figure~\ref{fig:visabi}, this causes severe degradation, increasing FID from $21$ to $217$ and reducing Sync-C from $8.38$ to $7.16$; visually, the generated identity drifts over time and the quality collapses into artifacts. This indicates that the identity-consistency improvement mainly comes from the  Brownian Bridge. Removing the SNR-aligned time transformation also harms fidelity (FID $378$) and produces blur, indicating that teacher-student noise alignment is necessary for heterogeneous DMD. Without audio-driven CFG, lip synchronization and motion quality drop substantially (Sync-C $4.34$, Sync-D $10.21$; Figure~\ref{fig:audiocfg}), showing its role in preserving fine-grained mouth dynamics under one-step inference.
\begin{table}
\centering
\small
\begin{tabular}{lccc}
\toprule
\textbf{Method} & \textbf{FID} & \textbf{Sync-C} & \textbf{Sync-D} \\
\midrule
LeapTalk                    & 21  & 8.38 & 7.69 \\
w/o Brownian Bridge      & 217 & 7.16 & 11.05 \\
w/o Time Transformation  & 378 & 7.84 & 8.13 \\
w/o Audio-Driven CFG     & 162 & 4.34 & 10.21 \\
\bottomrule
\end{tabular}
\caption{Ablation study results.}
\label{tab:ablation}
\end{table}

\begin{figure}
    \centering
    \includegraphics[width=1\linewidth]{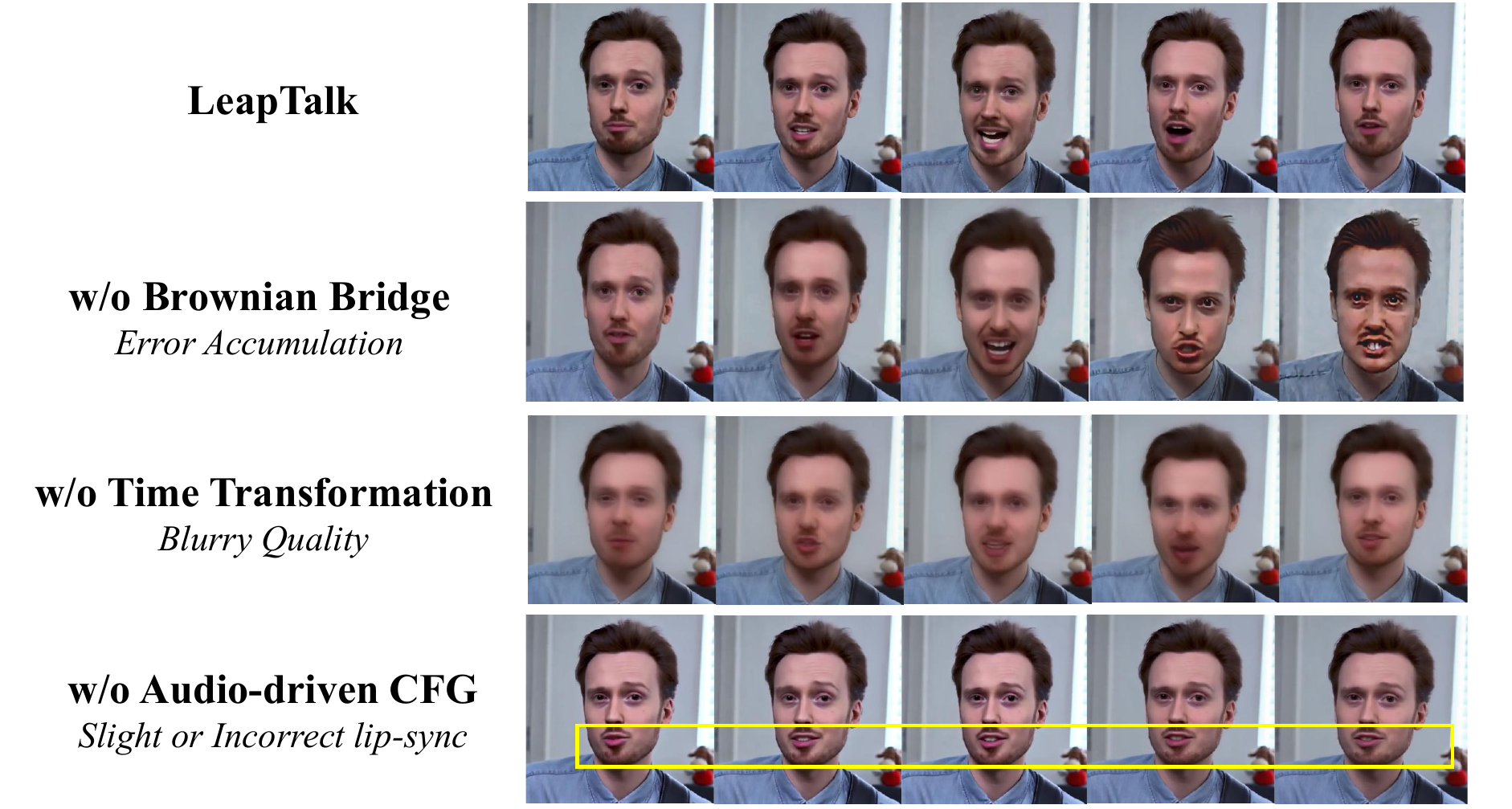}
    \caption{\textbf{Ablation results demonstrating the contribution of each component in our framework.}}
    \label{fig:visabi}
\end{figure}

\begin{figure}[h!]
    \centering
    \includegraphics[width=1\linewidth]{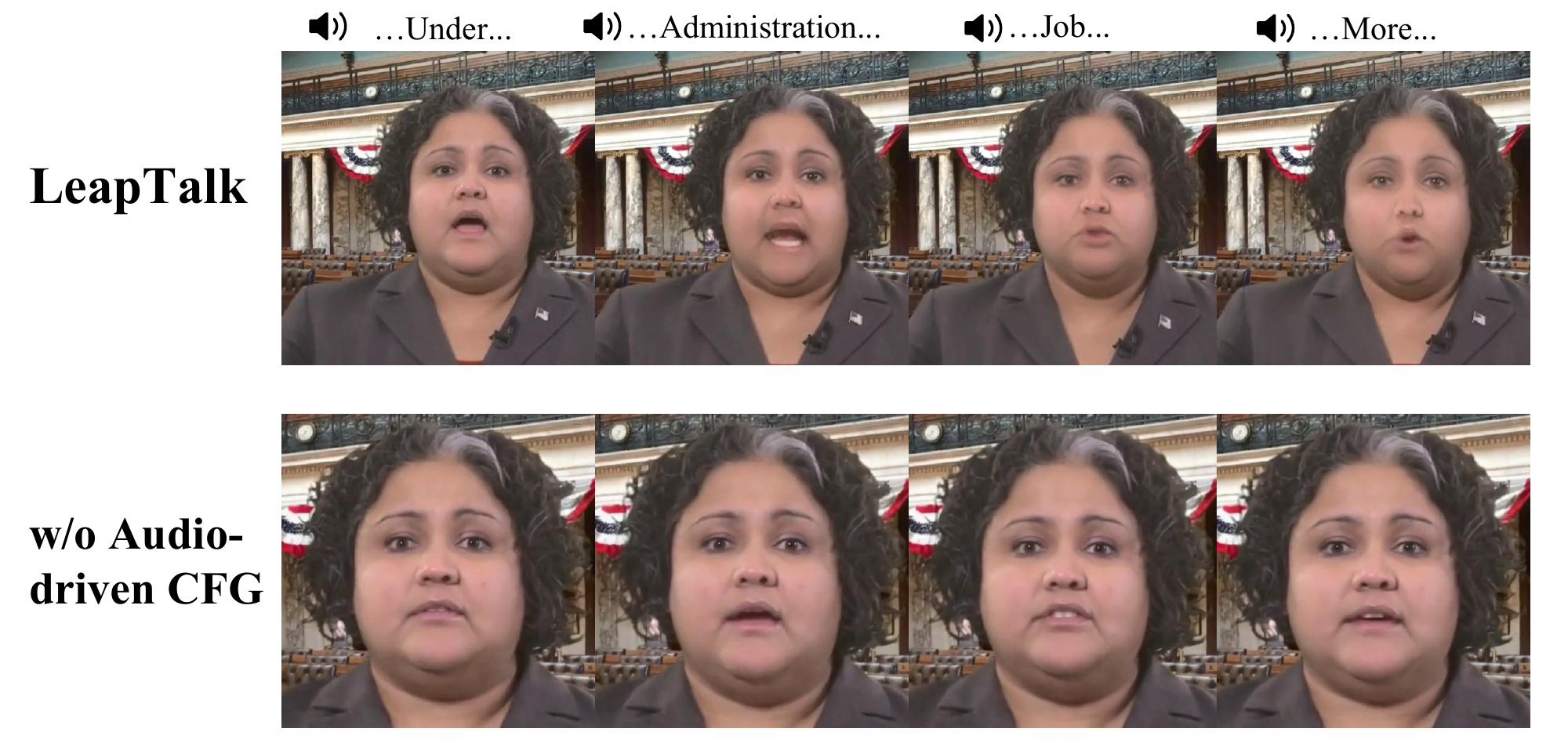}
    \caption{Visualization of ablation effects of audio-driven CFG.}
    \label{fig:audiocfg}
\end{figure}

\subsection{Parameter Sensitivity Analysis}
We analyze two key hyperparameters that control motion strength and reconstruction quality: the audio CFG scale $\alpha$ and the perceptual loss weight $\lambda_{\text{perc}}$ in Eq.~\ref{eq:final_loss}. As shown in Figure~\ref{fig:cfgsen}, increasing $\alpha$ improves lip-sync and motion expressiveness at first, but overly large guidance produces unnatural facial motion; a moderate value around $1.6$ achieves the best balance. Table~\ref{tab:lpipssensitivity} shows a similar non-monotonic trend for $\lambda_{\text{perc}}$: stronger perceptual supervision improves PSNR, SSIM, and LPIPS up to $\lambda_{\text{perc}}{=}4.0$, while $\lambda_{\text{perc}}{=}8.0$ degrades all metrics, indicating that excessive perceptual constraints destabilize visual quality. We therefore use $\alpha{=}1.6$ and $\lambda_{\text{perc}}{=}4.0$ as default settings. We also provide  analyses on chunk size and resolution in the supplementary material.
\begin{figure}[h!]
    \centering
    \includegraphics[width=1\linewidth]{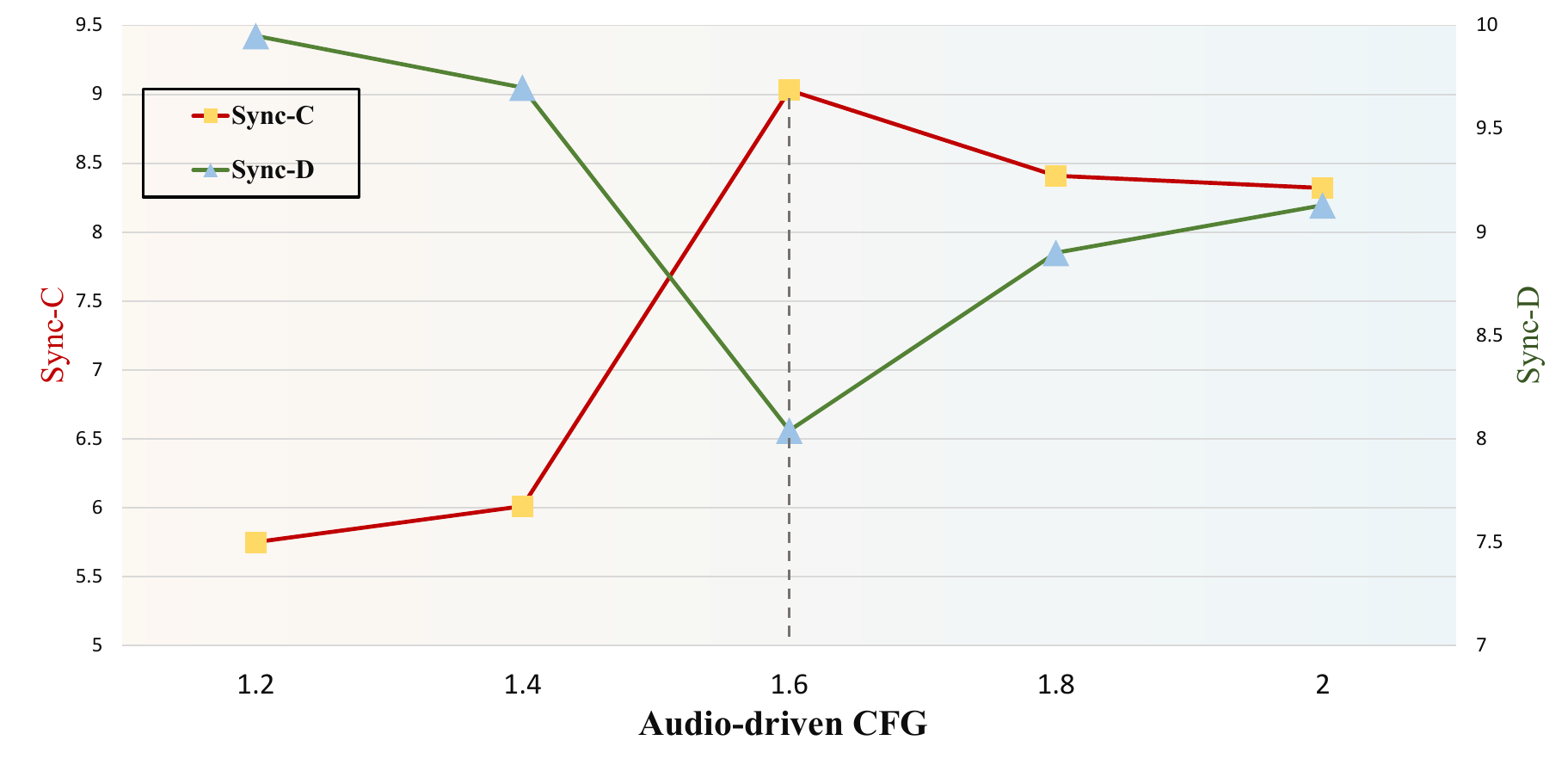}
    \caption{Sensitivity analysis of audio-driven CFG.}
    \label{fig:cfgsen}
\end{figure}
\begin{table}
\centering
\small
\setlength{\tabcolsep}{1.5pt}
\begin{tabular*}{\columnwidth}{@{\extracolsep{\fill}}lccccc@{}}
\toprule
Metric & $\lambda_{\text{perc}}{=}0$ & $\lambda_{\text{perc}}{=}1.0$ & $\lambda_{\text{perc}}{=}2.0$ & $\lambda_{\text{perc}}{=}4.0$ & $\lambda_{\text{perc}}{=}8.0$ \\
\midrule
PSNR$\uparrow$ &  18.62 & 19.18 & \underline{19.30} & \textbf{19.70} & 17.31 \\
SSIM$\uparrow$ &  0.625 & 0.683 & \underline{0.692} & \textbf{0.704} & 0.349 \\
LPIPS$\downarrow$ &  0.203 & 0.197 & \underline{0.194} & \textbf{0.183} & 0.562 \\
\bottomrule
\end{tabular*}
\caption{Perceptual loss weight sensitivity. \textbf{Bold}: best.}
\label{tab:lpipssensitivity}
\end{table}
\section{User Study and Further Discussions}
\label{sec:discussion}
\noindent\textbf{User Study.}
We conduct a user study with 30 participants, primarily consisting of university students and academic staff. In each question, participants are presented with videos generated by another method alongside our LeapTalk results, and are asked to select the best one based on four criteria: identity consistency, visual quality, lip-sync accuracy, and overall preference. As summarized in Table~\ref{tab:user}, LeapTalk is consistently preferred across all criteria, indicating stronger perceptual quality in addition to its quantitative performance.
\begin{table}[t]
\centering
\small
\setlength{\tabcolsep}{3pt}
\begin{tabularx}{\columnwidth}{@{}l*{4}{>{\centering\arraybackslash}X}@{}}
\toprule
\textbf{Method} & \textbf{\makecell{Identity\\Cons.}} & \textbf{\makecell{Lip-sync\\Acc.}} & \textbf{\makecell{Visual\\Quality}} & \textbf{\makecell{Overall\\Pref.}} \\
\midrule
StableAvatar     & 92.40\% & 93.14\% & 91.25\% & 91.37\% \\
Echomimic        & 91.26\% & 96.14\% & 94.36\% & 95.11\% \\
SoulX-FlashHead  & 95.23\% & 92.16\% & 91.89\% & 95.32\% \\
Hallo3           & 94.81\% & 95.48\% & 96.29\% & 96.21\% \\
FantasyTalking   & 96.89\% & 93.32\% & 95.36\% & 97.90\% \\
OmniAvatar       & 95.95\% & 94.31\% & 97.43\% & 95.75\% \\
\bottomrule
\end{tabularx}
\caption{User study results across different evaluation criteria. The table reports the percentage of participants who prefer our method to another method. A higher percentage indicates stronger user preference and better perceived performance.}
\label{tab:user}
\end{table}
\begin{table}[t]
\centering
\small
\setlength{\tabcolsep}{3pt}
\begin{tabularx}{\columnwidth}{@{}l>{\centering\arraybackslash}Xcccc@{}}
\toprule
\textbf{Autoenc.} & \textbf{Arch.} & 
\textbf{\makecell{Enc.\\Speed}} & 
\textbf{\makecell{Dec.\\Speed}} & 
\textbf{\makecell{Enc.\\Mem.}} & 
\textbf{\makecell{Dec.\\Mem.}} \\
\midrule
WanVAE & Causal Conv3D & 4.17s & 5.26s & 8.495GB & 10.128GB \\
TAEHV    & Conv2D        & \textbf{0.39s} & \textbf{0.24s} & \textbf{0.008GB} & \textbf{0.411GB} \\
\bottomrule
\end{tabularx}
\caption{\textbf{Comparison of autoencoder backbones used in Lite and Pro variants.} Encoding and decoding speeds are measured on video clips of 81 frames under BF16 precision. 
}
\label{tab:vae}
\end{table}
\begin{figure}
    \centering
    \includegraphics[width=1\linewidth]{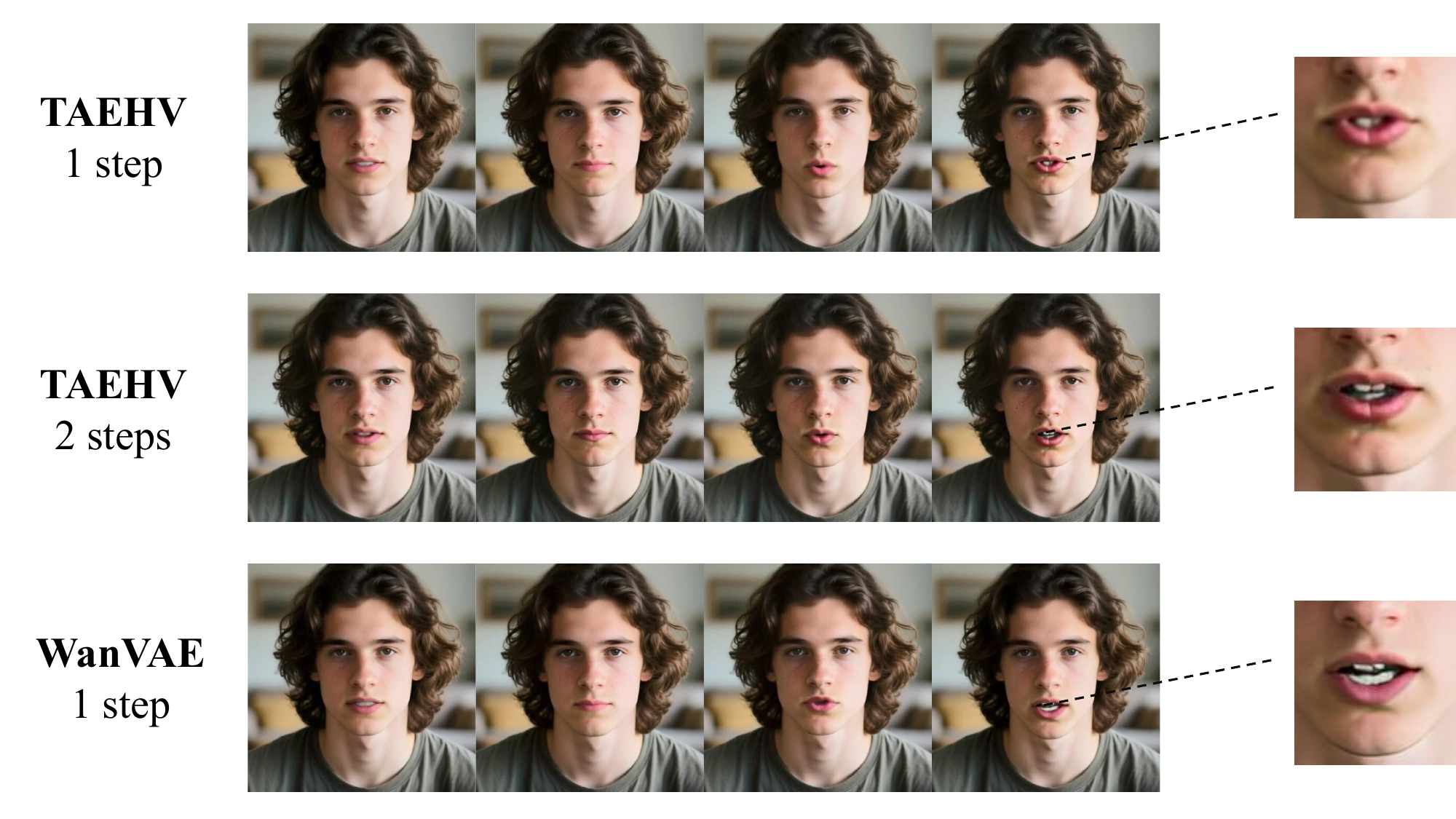}
    \caption{Visual comparison across different VAEs.}
    \label{fig:vae}
\end{figure}
\noindent\textbf{Effect of different Autoencoders.}
As shown in Figure~\ref{fig:vae} and Table~\ref{tab:vae}, TAEHV greatly reduces computation and memory compared with the Conv3D-based WanVAE, while preserving comparable structure, identity, and motion quality under one-step inference. Its main degradation is slight blurriness in fine regions such as lips, which can be alleviated by increasing inference from 1 to 2 steps, indicating a favorable efficiency--quality trade-off for real-time deployment.

\noindent\textbf{Generalization to Diverse Real-World Scenarios.}
We further evaluate our method under diverse and challenging inputs, including side-view portraits, cartoons, sculptures, and oil paintings. As shown in Figure~\ref{fig:np}, LeapTalk still synthesizes temporally coherent talking-head videos while preserving the reference identity and visual style. This suggests that the bridge-based formulation captures robust identity-motion correspondence beyond canonical frontal face images.
\begin{figure}[h!]
    \centering
    \includegraphics[width=1.0\linewidth]{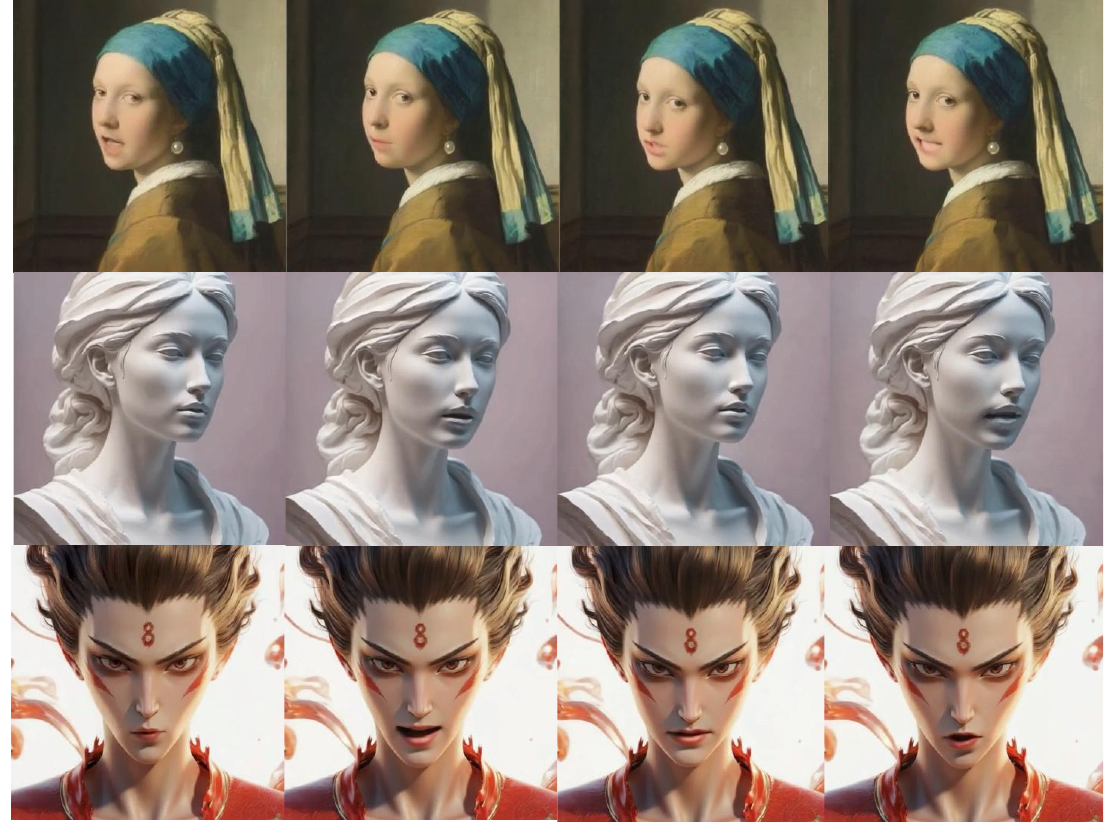}
    \caption{Generated results under diverse and challenging input scenarios.}
    \label{fig:np}
\end{figure}
\section{Conclusion}
We present LeapTalk, a novel framework for real-time, streaming, open-ended talking-head generation. By formulating generation as a Brownian bridge process, it reduces error accumulation and stabilizes long-term synthesis. A heterogeneous DMD scheme with SNR-aligned time transformation and audio-driven CFG augmentation enables effective distillation across mismatched paradigms and high-quality one-step generation. LeapTalk achieves competitive fidelity and lip synchronization at up to 200 FPS, breaking the latency–quality trade-off.

\bibliography{aaai2027}

\appendix

\section*{Appendix}

\section{Related Works}
\subsection{Long Video Generation}
Most video generation models~\citep{sora,zheng2024open,wan2025wan,gao2025wan} are limited to short clips (5--10 seconds), and extending them to longer durations without drift remains challenging. Early methods~\citep{qiu2023freenoise,kim2024fifo} rely on training-free noise rescheduling, while later works~\citep{chen2024diffusionforcingnexttokenprediction,ruhe2024rollingdiffusionmodels} simulate inference-time corruption during training. Other approaches explore architectural and training strategies such as next-frame prediction, causal attention, and rollout-based learning~\citep{zhang2025framepackv1,zhang2025framepack,huang2025selfforcing,zhu2026causalforcingautoregressivediffusion,li2025stablevideoinfinityinfinitelength,guo2025endtoendtrainingautoregressivevideo}. However, these methods still suffer from drift over long sequences. Moreover, they are primarily designed for text-to-video tasks and do not transfer well to audio-driven talking-head generation due to a significant domain gap.
\subsection{Audio-driven Talking Head Generation}
Audio-driven talking head generation has achieved remarkable visual fidelity with the advent of Diffusion Transformers (DiTs). However, DiT-based methods~\citep{chen2025echomimic,cui2025hallo3,gan2025omniavatar,wang2025fantasytalking,shen2023difftalk} can only generated a short clip and suffer from high latency due to the computational burden of multi-step denoising. While methods such as StableAvatar~\citep{tu2025stableavatar} enable infinite-length synthesis, they still operate entirely offline and cannot support interactive streaming. To bridge the latency gap, recent autoregressive approaches ~\citep{yu2026soulxflashheadoracleguidedgenerationinfinite,huang2025liveavatarstreamingrealtime,xiao2025knotforcingtamingautoregressive,wang2026restdiffusionbasedrealtimeendtoend} attempt high-speed streaming but inevitably suffer from error accumulation, identity drift, and detail degradation due to the flow-matching paradigm that starts from noise. 
\section{An Introduction to Brownian Bridge}
\label{sec:appendix_bb}

We briefly review the construction of the Brownian bridge~\citep{NEURIPS2021_940392f5,tan2025vision} starting from standard
Brownian motion and derive its closed-form expression.

A standard Brownian motion $\{B(t), t \ge 0\}$ is a stochastic process satisfying
$B(0)=0$, with independent increments, and such that
\begin{equation}
B(t)-B(s) \sim \mathcal{N}(0, t-s), \quad 0 \le s < t.
\end{equation}
As a consequence, $B(t) \sim \mathcal{N}(0,t)$ and
\begin{equation}
\mathrm{Cov}(B(s), B(t)) = \min(s,t).
\end{equation}
For a given initial point $x$, we define a Brownian motion starting from $x$ as
$B^x(t) = x + B(t)$.

Let $\{B(t), t \ge 0\}$ be a standard Brownian motion, i.e., $B(0)=0$,
it has independent increments, and $B(t)-B(s)\sim \mathcal{N}(0, t-s)$
for $0 \le s < t$. For a given initial point $x$, we define a Brownian motion
starting from $x$ as $B^x(t) = x + B(t)$.

We consider the process starting from $X_{\mathrm{src}}$:
\begin{equation}
X_\tau = X_{\mathrm{src}} + B(\tau), \quad \tau \in [0,1],
\end{equation}
which satisfies $X_0 = X_{\mathrm{src}}$ and
\begin{equation}
X_\tau \sim \mathcal{N}(X_{\mathrm{src}}, \tau I).
\end{equation}

However, this process does not constrain the terminal value $X_1$. To enforce
$X_1 = X_{\mathrm{tgt}}$, we consider the conditioned process
\begin{equation}
X_\tau \mid_{ \tau=1} = X_{\mathrm{tgt}},
\end{equation}
which defines the Brownian bridge.

Since Brownian motion is a Gaussian process, the joint distribution of
$(X_\tau, X_1)$ is Gaussian. In particular,
\begin{align}
X_\tau &= X_{\mathrm{src}} + B(\tau), \\
X_1 &= X_{\mathrm{src}} + B(1),
\end{align}
and their covariance satisfies
\begin{equation}
\mathrm{Cov}(B(\tau), B(1)) = \tau.
\end{equation}

By standard Gaussian conditioning, we have
\begin{equation}
\mathbb{E}[X_\tau \mid X_1]
= X_{\mathrm{src}} + \tau (X_1 - X_{\mathrm{src}}),
\end{equation}
which yields
\begin{equation}
\mathbb{E}[X_\tau \mid X_1 = X_{\mathrm{tgt}}]
= (1-\tau) X_{\mathrm{src}} + \tau X_{\mathrm{tgt}}.
\end{equation}

The conditional variance is given by
\begin{equation}
\mathrm{Var}(X_\tau \mid X_1)
= \tau(1-\tau) I.
\end{equation}

Therefore, the conditioned process admits the representation
\begin{equation}
X_\tau = (1-\tau) X_{\mathrm{src}} + \tau X_{\mathrm{tgt}} + \sqrt{\tau(1-\tau)}\,\epsilon,
\quad \epsilon \sim \mathcal{N}(0,I),
\end{equation}
which is the Brownian bridge connecting $X_{\mathrm{src}}$ and
$X_{\mathrm{tgt}}$.

In our formulation, we adopt an equivalent time-reversed parameterization
\begin{equation}
X_\tau = \tau X_{\mathrm{src}} + (1-\tau) X_{\mathrm{tgt}} + \sqrt{\tau(1-\tau)}\,\epsilon,
\end{equation}
so that $\tau=1$ corresponds to the source state and $\tau=0$ corresponds to
the target state.
\section{Derivation of the SNR-aligned Time Transformation}
\label{sec:appendix_snr}

In this section, we derive the time transformation $t = \Phi(\tau)$ that aligns
the signal-to-noise ratios (SNRs) between the teacher and student processes.
This alignment ensures that both models operate under equivalent noise levels,
which is essential for consistent score matching in heterogeneous distillation.


\noindent\textbf{Proof.} 
For the student model, we consider the Brownian bridge formulation
\begin{equation}
X_\tau = \tau X_{\mathrm{src}} + (1-\tau) X_{\mathrm{tgt}} + \sqrt{\tau(1-\tau)}\,\epsilon,
\end{equation}
with $\epsilon \sim \mathcal{N}(0,I)$ and $\tau \in (0,1)$. Since the objective
is to recover $X_{\mathrm{tgt}}$, we take $(1-\tau)X_{\mathrm{tgt}}$ as the signal
component and treat $\tau X_{\mathrm{src}}$ as a deterministic bias term. The
resulting SNR is therefore
\begin{equation}
\mathrm{SNR}_{\mathrm{Student}}(\tau) = \frac{(1-\tau)^2}{\tau(1-\tau)} = \frac{1-\tau}{\tau}.
\end{equation}

For the teacher, we adopt the diffusion parameterization
\begin{equation}
x_t = (1-t)x_0 + t\epsilon, \quad t \in (0,1),
\end{equation}
which yields
\begin{equation}
\mathrm{SNR}_{\mathrm{Teacher}}(t) = \frac{(1-t)^2}{t^2}.
\end{equation}

We align the two processes by enforcing $\mathrm{SNR}_{\mathrm{Teacher}}(t)
= \mathrm{SNR}_{\mathrm{Student}}(\tau)$, which gives
\begin{equation}
\frac{(1-t)^2}{t^2} = \frac{1-\tau}{\tau}.
\end{equation}
Solving for $t \in (0,1)$ yields a unique monotonic mapping
\begin{equation}
t = \Phi(\tau) = \frac{1}{1 + \sqrt{\frac{1-\tau}{\tau}}}.
\end{equation}

This mapping preserves the endpoints $\tau \to 0 \Rightarrow t \to 0$ and
$\tau \to 1 \Rightarrow t \to 1$, and satisfies $\Phi(0.5)=0.5$, aligning the
maximal uncertainty point across the two processes. By construction, $\Phi(\tau)$
ensures that teacher and student operate at equivalent noise levels, which
provides a consistent basis for score matching despite their heterogeneous
generative formulations.

\section{Comparison on Motion Diversity and Audio CFG Ablation}
To measure the pose variance and naturalness, we add Hopenet-based pose diversity~\cite{Doosti_2020_CVPR} (Yaw/Pitch/Roll Std and Average Std) and Beat Align Score (BAS) to verify this. As shown in Tables~\ref{tab:motion_baseline}--\ref{tab:motion_cfg}, \textbf{LeapTalk achieves superior or competitive diversity and alignment vs.\ baselines.}\\
Table~\ref{tab:motion_cfg} directly validates the audio CFG claim: pose diversity (Avg Std) increases with CFG scale, \textbf{proving CFG effectively enhances motion and avoids static behavior.} However, BAS gradually drops as CFG grows, indicating over-strong guidance can hurt audio-motion alignment. 
\begin{table}[h]
\centering
\resizebox{1\linewidth}{!}{
\begin{tabular}{lccccr}
\toprule
Method & Yaw Std$\uparrow$ & Pitch Std$\uparrow$ & Roll Std$\uparrow$ & Avg Std$\uparrow$ & BAS$\uparrow$ \\
\midrule
EchoMimic       & 1.734 & 1.867 & 0.771 & 1.457 & 0.650 \\
OmniAvatar      & \underline{4.332} & \textbf{5.888} & 1.909 & \underline{4.043} & 0.652 \\
SoulX-Flashhead & 3.594 & 2.918 & \underline{1.959} & 2.824 & \underline{0.684} \\
\textbf{Ours}   & \textbf{6.177} & \underline{5.632} & \textbf{2.146} & \textbf{4.652} & \textbf{0.696} \\
\bottomrule
\end{tabular}}
\caption{Motion diversity vs.\ baselines on the  HDTF samples. \textbf{Bold}: best; \underline{uline}: 2nd best.}
\label{tab:motion_baseline}
\end{table}
\begin{table}[h]
\centering
\resizebox{1\linewidth}{!}{
\begin{tabular}{cccccc}
\toprule
CFG Scale & Yaw Std$\uparrow$ & Pitch Std$\uparrow$ & Roll Std$\uparrow$ & Avg Std$\uparrow$ & BAS$\uparrow$ \\
\midrule
1.0 & 1.553 & 2.653 & 0.757 & 1.655 & 0.723 \\
3.0 & 4.021 & 4.493 & 1.668 & 3.394 & 0.658 \\
5.0 & 6.840 & 5.764 & 2.396 & 5.000 & 0.696 \\
7.0 & 8.533 & 7.635 & 2.801 & 6.323 & 0.650 \\
\bottomrule
\end{tabular}}
\caption{Audio CFG scale ablation on HDTF.}
\label{tab:motion_cfg}
\end{table}
\section{Streaming Chunk Size and Playback Smoothness:}
\begin{table}[h]
\centering
\resizebox{1\linewidth}{!}{
\begin{tabular}{cccccc}
\toprule
Chunk Size & $T_{\mathrm{chunk}}$ (s) & $T_{\mathrm{gen}}$ (s) & $T_{\mathrm{gen}}/T_{\mathrm{chunk}}$ $\downarrow$ & FPS $\uparrow$ \\
\midrule
9  & 0.36 & 0.088 & 0.24 & 45.4 \\
17 & 0.68 & 0.146 & 0.22 & 82.0 \\
\textbf{33} & \textbf{1.32} & \textbf{0.268} & \textbf{0.20} & \textbf{104.7} \\
49 & 1.96 & 0.432 & 0.22 & 101.9 \\
65 & 2.60 & 0.629 & 0.24 & 95.3 \\
\bottomrule
\end{tabular}}
\caption{Chunk size ablation at $512\times512$, 1-step inference, single A100 GPU. $T_{\mathrm{gen}}/T_{\mathrm{chunk}}{<}1$ indicates real-time. \textbf{Bold} represents chosen default, which achieves the best ratio.}
\label{tab:chunk_size}
\end{table}
Our chunk size is \textbf{33 frames}. For streaming playback, chunk $i{+}1$ is generated while chunk $i$ is being played; therefore, smooth playback only requires $T_{\mathrm{gen}}/T_{\mathrm{chunk}} < 1$. As shown in Table~\ref{tab:chunk_size}, \textbf{33 frames achieves the best $T_{\mathrm{gen}}/T_{\mathrm{chunk}}$ ratio (0.20), meaning the next chunk is generated within only 20\% of the current chunk's playback time, leaving sufficient margin for continuous gap-free streaming.}\\

\section{Resolution and Inference Speed:}
\textbf{All main experiments were conducted at $512 \times 512$ resolution.} Moreover, Table \ref{tab:resolution_latency} provides additional results reporting FPS and latency under different resolutions on a single A100 GPU. 
Note that the 200 fps in the main manuscript is measured on H200. We will align the hardware in the revision. 
\begin{table}[h]
\centering
\resizebox{1\linewidth}{!}{
\begin{tabular}{c|ccccc}
\toprule
Resolution & $256\times256$ & $384\times384$ & $\mathbf{512\times512}$ & $768\times768$ & $1024\times1024$ \\
\midrule
FPS $\uparrow$           & 476 & 191 & {104} & 35 & 14 \\
Latency (s) $\downarrow$ & 0.059 & 0.146 & {0.270} & 0.793 & 1.908 \\
\bottomrule
\end{tabular}}
\caption{Inference speed and average chunk generation latency under different resolutions on a single A100 GPU (1-step inference).}
\label{tab:resolution_latency}
\end{table}
\section{Additional Results in Challenging Scenarios}
\label{sec:appendix_more_results}
We provide additional qualitative results to evaluate our framework under challenging and out-of-distribution conditions, including \textbf{non-human faces (e.g., animals)}, \textbf{artistic portraits (e.g., paintings and stylized illustrations)}, \textbf{side-view inputs}, \textbf{low-light scenarios}, \textbf{partial occlusions}, and \textbf{non-photorealistic objects such as sculptures}. Despite significant variations in geometry, texture, and illumination, our Bridge Forcing framework consistently produces stable and coherent animations, preserving identity cues anchored to $\mathcal{I}$ while maintaining smooth motion dynamics. Notably, even under severe appearance gaps, the model avoids long-horizon drift, demonstrating the effectiveness of the endpoint-constrained data-to-data formulation and autoregressive chunk design. Additional results are shown in Fig.~\ref{fig:more_results}.

\begin{figure}[ht!]
    \centering
    \includegraphics[width=1\linewidth]{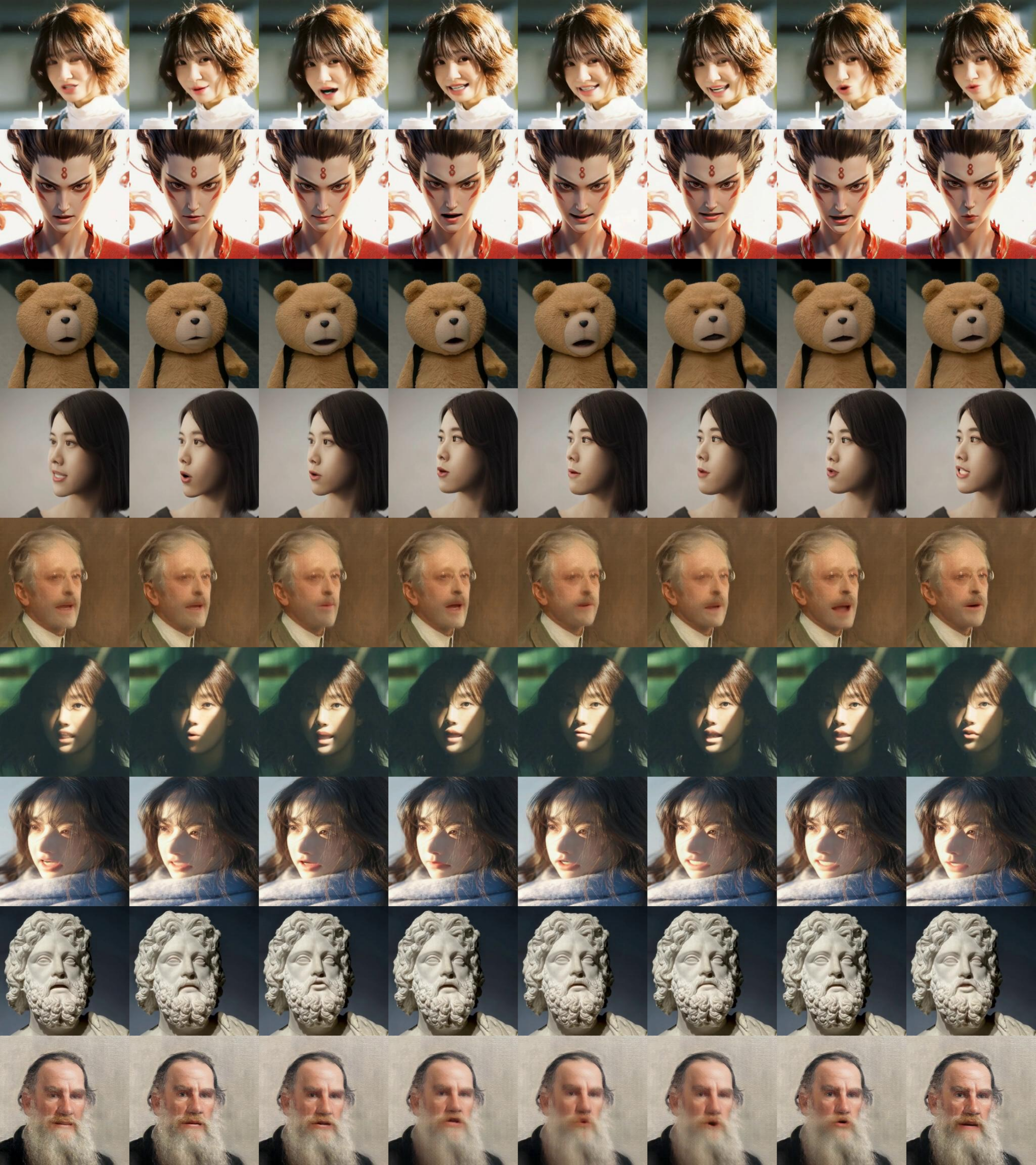}
    \caption{More qualitative results on challenging conditions including \textit{non-human faces (e.g., animals)}, \textit{artistic portraits (e.g., paintings and stylized illustrations)},\textit{ side-view photos},\textit{ low-light conditions}, \textit{partial occlusions,} and even \textit{non-photorealistic objects such as sculptures}.}
    \label{fig:more_results}
\end{figure}
\section{Training Algorithm of heterogeneous DMD}
\label{sec:appendix_training}
Algorithm~\ref{alg:snaptalk_training} shows the training algorithm of heterogeneous DMD. The training procedure consists of three main steps: (1) generating a student chunk and a teacher target, (2) computing the generator losses, and (3) updating the generator and fake score estimation model. The generator is trained to minimize the distribution matching loss, bridge loss, and LPIPS loss, while the fake score estimation model is updated to improve its ability to estimate the score of the generated samples.
\begin{algorithm}
\caption{Training Procedure with Heterogeneous DMD}
\label{alg:snaptalk_training}
\begin{algorithmic}[1]
\REQUIRE Pretrained flow-matching teacher $V_{\text{real}}$, paired dataset $\mathcal{D} = \{(I, \mathbf{c})\}$
\REQUIRE Audio CFG scale $s$, spatial loss weights $\lambda_{\text{face}}, \lambda_{\text{lip}}$
\ENSURE Trained single-step generator $G_\theta$

\STATE // \textit{Initialize generator and fake score estimators from pretrained model}
\STATE $G_\theta \leftarrow \text{copyWeights}(V_{\text{real}})$, $V_{\text{fake}} \leftarrow \text{copyWeights}(V_{\text{real}})$
\WHILE{train}
    \STATE Sample batch $(I, \mathbf{c}_{1:K}) \sim \mathcal{D}$ \quad // \textit{$I$: reference, $\mathbf{c}$: audio chunks}
    \STATE $\mathbf{h}_1 \leftarrow I$ \quad // \textit{Initialize autoregressive history}
    \FOR{chunk $k = 1$ \TO $K$}
        \STATE // \textit{Generate student chunk and teacher target}
        \STATE $\hat{\mathbf{x}}_k \leftarrow G_\theta(\mathbf{h}_k, I, \mathbf{c}_k)$ \quad // \textit{1-step generation}
        \STATE $\mathbf{x}^*_k \leftarrow \text{multiStepRollout}(V_{\text{real}}, \mathbf{h}_k, I, \mathbf{c}_k)$ \quad // \textit{Teacher target}
        
        \STATE // \textit{Update Generator}
        \STATE Sample student bridge time $\tau \sim \mathcal{U}(0, 1)$ and $\epsilon \sim \mathcal{N}(0, \mathbf{I})$
        \STATE $\mathbf{x}_\tau \leftarrow (1-\tau)I + \tau\hat{\mathbf{x}}_k + \sqrt{\tau(1-\tau)}\epsilon$ \quad // \textit{Bridge trajectory}
        \STATE $t \leftarrow \Phi(\tau) = 1 / (1 + \sqrt{(1-\tau)/\tau})$ \quad // \textit{SNR-aligned Time Transform}
        
        \STATE // \textit{Audio-driven Classifier-Free Guidance for Teacher}
        \STATE $v_{\text{cond}} \leftarrow V_{\text{real}}(\mathbf{x}_t, t, \mathbf{c}_k, I)$
        \STATE $v_{\text{uncond}} \leftarrow V_{\text{real}}(\mathbf{x}_t, t, \emptyset, I)$
        \STATE $v_{\text{real}}^{\text{cfg}} \leftarrow v_{\text{uncond}} + s \cdot (v_{\text{cond}} - v_{\text{uncond}})$
        
        \STATE // \textit{Compute Generator Losses}
        \STATE $\mathbf{x}_{\text{real}} \leftarrow \mathbf{x}_t - t \cdot v_{\text{real}}^{\text{cfg}}$ \quad // \textit{Flow-matching to data}
        \STATE $\mathbf{x}_{\text{fake}} \leftarrow \mathbf{x}_\tau - \tau \cdot V_{\text{fake}}(\mathbf{x}_\tau, \tau, \mathbf{c}_k, I)$ \quad // \textit{Bridge to data}
        
        \STATE $W \leftarrow 1 + \lambda_{\text{face}}M_{\text{face}} + \lambda_{\text{lip}}M_{\text{lip}}$ \quad // \textit{Spatial weights}
        \STATE $\mathcal{L}_{\text{DMD}} \leftarrow \text{distributionMatchingLoss}(\mathbf{x}_{\text{real}}, \mathbf{x}_{\text{fake}}, W)$
        \STATE $\mathcal{L}_{\text{Bridge}} \leftarrow \| W \odot (\hat{\mathbf{x}}_k - \mathbf{x}^*_k) \|_2^2$ \quad // \textit{Weighted MSE}
        \STATE $\mathcal{L}_{\text{LPIPS}} \leftarrow \text{LPIPS}(\hat{\mathbf{x}}_k, \mathbf{x}^*_k)$
        \STATE $\mathcal{L}_G \leftarrow \mathcal{L}_{\text{DMD}} + \lambda_{\text{br}}\mathcal{L}_{\text{Bridge}} + \lambda_{\text{lpips}}\mathcal{L}_{\text{LPIPS}}$
        \STATE $G_\theta \leftarrow \text{update}(G_\theta, \mathcal{L}_G)$
        
        \STATE // \textit{Update fake score estimation model (Critic)}
        \STATE Sample time step $\tau' \sim \mathcal{U}(0, 1)$
        \STATE $\tilde{\mathbf{x}} \leftarrow \text{stopgrad}(\hat{\mathbf{x}}_k)$
        \STATE $\mathbf{x}_{\tau'} \leftarrow (1-\tau')I + \tau'\tilde{\mathbf{x}} + \sqrt{\tau'(1-\tau')}\epsilon$
        \STATE $v_{\text{target}} \leftarrow (\mathbf{x}_{\tau'} - \tilde{\mathbf{x}}) / \tau'$ \quad // \textit{Target bridge velocity}
        \STATE $\mathcal{L}_{\text{critic}} \leftarrow \| V_{\text{fake}}(\mathbf{x}_{\tau'}, \tau', \mathbf{c}_k, I) - v_{\text{target}} \|_2^2$ 
        \STATE $V_{\text{fake}} \leftarrow \text{update}(V_{\text{fake}}, \mathcal{L}_{\text{critic}})$
        
        \STATE // \textit{Autoregressive Update}
        \STATE $\mathbf{h}_{k+1} \leftarrow \text{stopgrad}(\hat{\mathbf{x}}_k)$ \quad // \textit{Detach history chunks}
    \ENDFOR
\ENDWHILE
\end{algorithmic}
\end{algorithm}


\end{document}